\documentclass[12pt]{arxiv} 

\newcommand{\Real}{\mathbb R}
\newcommand{\eps}{\varepsilon}
\newcommand{\abs}[1]{\left\vert#1\right\vert}
\newcommand{\norm}[1]{\left\Vert#1\right\Vert}
\newcommand{\ra}{\rightarrow}
\usepackage{mathtools}
\usepackage{multirow}
\DeclarePairedDelimiter\inner{\langle}{\rangle}
\DeclareMathOperator{\E}{\mathbb{E}}
\renewcommand{\set}[1]{\left\{#1\right\}}

\title[On Almost Sure Convergence Rates of Stochastic Gradient Methods]{On Almost Sure Convergence Rates of Stochastic Gradient Methods}

\usepackage{times}

\author{%
 \Name{Jun Liu} \Email{j.liu@uwaterloo.ca}\\
 \addr Department of Applied Mathematics, University of Waterloo, Waterloo, Canada
 \AND
 \Name{Ye Yuan} \Email{yye@hust.edu.cn}\\
 \addr School of Artificial Intelligence and Automation \& School of
 Mechanical Science and Engineering, Huazhong University of Science and Technology, Wuhan, China
}
\begin{document}
	
	\maketitle
	
	\thispagestyle{plain}
	
	\begin{abstract}%
		The vast majority of convergence \textit{rates} analysis for stochastic gradient methods in the literature focus on convergence in expectation, whereas trajectory-wise almost sure convergence is clearly important to ensure that \textit{any instantiation} of the stochastic algorithms would converge with probability one. Here we provide a unified almost sure convergence rates analysis for stochastic gradient descent (SGD), stochastic heavy-ball (SHB), and stochastic Nesterov's accelerated gradient (SNAG) methods.  We show, for the first time, that the almost sure convergence rates obtained for these stochastic gradient methods on strongly convex functions, are arbitrarily close to their optimal convergence rates possible. For non-convex objective functions, we not only show that a weighted average of the squared gradient norms converges to zero almost surely, but also the \textit{last iterates} of the algorithms. We further provide \textit{last-iterate} \textit{almost sure} convergence rates analysis for stochastic gradient methods on general convex smooth functions, in contrast with most existing results in the literature that only provide convergence in expectation for a weighted average of the iterates. 
	\end{abstract}
	
	
	\section{Introduction}
	
	Stochastic gradient methods \citep{robbins1951stochastic} have become the \textit{de facto} standard methods for solving large-scale optimization problems in machine learning \citep{bottou2018optimization}. For this reason, investigating the fundamental theoretical properties of stochastic gradient methods is not only of theoretical interest, but also of practical relevance. 
	
	Stochastic gradient descent (SGD) \citep{robbins1951stochastic} and stochastic heavy-ball (SHB) \citep{polyak1964some} are among the most popular stochastic gradient methods. SHB adds a momentum term to the iterations of SGD. This was known to accelerate the convergence of deterministic gradient descent methods \citep{polyak1964some}. Nesterov's accelerated gradient (NAG) methods  \citep{nesterov1983method} have similar but slightly different iterations from that of the heavy-ball (HB) method. They have been shown to accelerate gradient descents and achieve optimal convergence rates with appropriately chosen parameters in the deterministic settings \citep[Chapter 2.2]{nesterov2003introductory}. In the stochastic settings, while practical gains of adding a momentum term have been observed \citep{leen1994optimal,sutskever2013importance}, the convergence rates cannot be further improved due to the proven lower bounds in terms of oracle complexity \citep{agarwal2012information}. Nonetheless, understanding the convergence properties of stochastic gradient methods with or without momentum remains a topic of both theoretical and practical interest. 
	
	In this paper, we investigate almost sure convergence properties of stochastic gradient methods, including SGD, SHB, and stochastic Nesterov's accelerated gradient (SNAG) methods, and present a unified analysis of these stochastic gradient methods on smooth objective functions. 
	

	
	\subsection{Related work} 
	
	The vast majority of the convergence rates analysis results for stochastic gradient methods in the literature are obtained in terms of the expectation (see, e.g., SGD \citep{nemirovski2009robust,moulines2011non,ghadimi2013stochastic}, SHB \citep{yang2016unified,orvieto2020role,yan2018unified,mai2020convergence,zhou2020pbsgd}, SNAG \citep{yan2018unified,assran2020convergence,laborde2020lyapunov}). Nonetheless, almost sure convergence properties are important, because they represent what happen to individual trajectories of the stochastic iterations, which are instantiations of the stochastic algorithms actually used in practice.  
	
	For this reason, almost sure convergence of stochastic gradient methods is of practical relevance. In fact, the early analysis of SGD \citep{robbins1971convergence} did provide almost sure convergence guarantees. More recent work includes  \cite{bertsekas2000gradient,bottou2003stochastic,zhou2017stochastic,nguyen2018sgd,nguyen2019new,orabona2020almost,mertikopoulos2020almost}. While deterministic HB and NAG methods are well analyzed \citep{ghadimi2015global,nesterov2003introductory,wilson2016lyapunov}, almost sure convergence results for SHB and SNAG are scarce. \cite{gadat2018stochastic} proved almost sure convergence of SHB to a minimizer for non-convex functions, under a uniformly elliptic condition on the noise which helps the algorithm to get out of any unstable point. In \cite{sebbouh2021almost}, SHB (and SGD) was analyzed for convex (but not strongly convex or non-convex) objective functions. The authors not only proved almost sure convergence for iterations of SGD and SHB, they also established convergence rates that are close to optimal (subject to an $\eps$-factor in the rate) for general convex functions. Almost sure convergence rates were analyzed for SGD under locally strongly convex objectives in \cite{pelletier1998almost,godichon2019lp}. To the best knowledge of the authors, the results in \cite{sebbouh2021almost} are the only ones that established almost sure convergence \textit{rates} for SHB on general convex functions. We are not aware of any almost sure convergence \textit{rates} analysis for SHB and SNAG on strongly convex or non-convex functions. The results of this paper aim to fill this theoretical gap. 
	
	\subsection{Contributions} 
	
	We summarize the main contributions of the paper in Table \ref{tab:results} relative to exiting results in the literature. We only list results that provided \textit{almost sure convergence rates} analysis for SGD, SHB, and SNAG. We emphasize the following results as the main contributions:
	
	\begin{itemize}
		\item For smooth and strongly convex functions, we establish almost sure convergence rates for SGD, SHB and SNAG that are arbitrarily close to the optimal rates possible. 
		\item For smooth but non-convex functions, we establish almost sure convergence rates of SHB and SNAG for a weighted average (or the minimum) of the squared gradient norm. We also show almost sure convergence of the last iterates of SHB and SNAG.  
		\item For smooth and general convex functions, we provide almost sure convergence rates of the last iterates of SGD, SHB, and SNAG. 
	\end{itemize}
	
	In view of existing results \cite{pelletier1998almost,godichon2019lp}, our analysis for almost sure convergence rates analysis of SGD on strongly convex functions appears to be more streamlined and unified for SGD, SHB, and SNAG. For analysis of SHB in the general convex case, our result is complementary to that in \cite{sebbouh2021almost} because we allow $\beta$ to be an arbitrarily fixed parameter in $(0,1)$  (cf. the analysis of deterministic HB in \cite{ghadimi2015global}). This leads to a more unified analysis of SGD, SHB, and SNAG. In addition to the results listed in Table \ref{tab:results}, we also obtained another set of results (Theorem \ref{thm:last-iterate}) on almost sure convergence of the last iterates of SHB and SNAG on non-convex functions, which generalize \cite{orabona2020almost} for SGD. 
	
	\begin{table}[h!]
		\centering
		\begin{tabular}{ |c||c|c|c|c| } 
			\hline
			\textbf{Algorithm} & strongly convex & non-convex & general convex\\
			\hline
			\multirow{3}{2em}{SGD}	& \cite{pelletier1998almost} & \cite{sebbouh2021almost} & \cite{sebbouh2021almost}\\ 
			& \cite{godichon2019lp} &  & \\ 
			& Theorem \ref{thm:sgd}  & Theorem \ref{thm:sgd} & Theorem \ref{thm:last-iterate-rates}\\ 
			\hline 
			\multirow{2}{2em}{SHB} & Theorem \ref{thm:shb} & Theorem \ref{thm:shb} & \cite{sebbouh2021almost} \\ 
			& & & Theorem \ref{thm:last-iterate-rates} \\
			\hline 
			SNAG & Theorem \ref{thm:snag} & Theorem \ref{thm:snag} & Theorem \ref{thm:last-iterate-rates} \\ 
			\hline
		\end{tabular}
		\caption{Summary of the main results relative to existing results on \textit{almost sure} convergence \textit{rates} of stochastic gradient methods.}\label{tab:results}
	\end{table}

	\subsection{Problem statement and assumptions}
	We are interested in solving the unconstrained minimization problem 
	\begin{equation}
		\label{eq:min}
		\min_{x\in\Real^d} f(x), 
	\end{equation}
	where $f:\,\Real^d\ra\Real$, using stochastic gradient methods. For example, with a slight abuse of notation, $f$ may arise from optimizing an expected risk of the form $f(x)=\E[f(x;\xi)]$, where $\xi$ is a source of randomness indicating a sample (or a set of samples), or an empirical risk of the form $f(x)=\frac{1}{n}\sum_{i=1}^nf_i(x;\xi_i)$, where $\set{\xi_i}_{i=1}^n$ are realizations of $\xi$ \citep{bottou2018optimization}. We make the following assumptions. 
	
	\begin{assumption}[$L$-smoothness]\label{as:smoothness}
		The continuously differentiable function $f:\,\Real^d\ra\Real$ is bounded from below by $f^*:=\inf_{x\in\Real^d}f(x)\in \Real$ and its gradient $\nabla f$ is $L$-Lipschitz, i.e., $\norm{\nabla f(x)-\nabla f(y)}\le L\norm{x-y}$ for all $x,y\in\Real^d$. 
	\end{assumption}
	A useful consequence of Assumption \ref{as:smoothness} (see, e.g., \citet[Lemma 1.2.3]{nesterov2003introductory}) is the following 
	\begin{equation}
		\label{eq:L}
		f(y) \le f(x) + \inner{\nabla f(x),y-x} + \frac{L}{2}\norm{y-x}^2,\quad \forall x,y\in\Real^d.
	\end{equation}
	
	In some settings, we also assume that $f$ is strongly convex. 
	\begin{assumption}[$\mu$-strongly convex]\label{as:strong}
		There exists a positive constant $\mu$ such that 
		$$
		f(y) \ge f(x) + \inner{\nabla f(x),y-x} + \frac{\mu}{2}\norm{y-x}^2,\quad \forall x,y\in\Real^d.
		$$ 
	\end{assumption}
	
	Assumption \ref{as:strong} with $\mu=0$ will be referred to as general convexity. 
	When $f$ is convex (strongly or generally), we further assume that $f$ has a minimizer, i.e., $x_*\in\Real^d$ such that $f^*=f(x_*)$. A consequence of $f$ being $\mu$-strongly convex is that (see, e.g., \citet[Theorem 2.1.10]{nesterov2003introductory})
	\begin{equation}
		\label{eq:fmu}
		\frac{1}{2\mu}\norm{\nabla f(x)}^2\ge f(x)-f^*,\quad \forall x\in\Real^d.
	\end{equation}
	In contrast, if $f$ is generally convex and $L$-smooth, we have (see, e.g., \citet[Theorem 2.1.5]{nesterov2003introductory})
	\begin{equation}
		\label{eq:Lf}
		\frac{1}{2L}\norm{\nabla f(x)}^2 \le f(x)-f^*,\quad \forall x\in\Real^d. 
	\end{equation}
	
	Since we are interested in solving (\ref{eq:min}) using stochastic gradient methods, we assume at each $x\in\Real^d$, we have access to an unbiased estimator of the true gradient $\nabla f(x)$, denoted by $\nabla f(x;\xi)$. 
	
	\begin{assumption}[ABC condition]\label{as:abc}
		There exist nonnegative constants $A$, $B$, and $C$ such that 
		\begin{equation}\label{eq:abc}
			\E[\norm{\nabla f(x;\xi)}^2] \le A(f(x)-f^*) + B\norm{\nabla f(x)}^2 + C,\quad \forall x\in \Real^d. 
		\end{equation}
	\end{assumption}
	
	\begin{remark}
		The above assumption was proposed in \cite{khaled2020better} as ``the weakest assumption'' for analysis of SGD in the non-convex setting. This assumption clearly includes the uniform bound 
		$
		\E[\norm{\nabla f(x;\xi)}^2] \le \sigma^2
		$
		and bounded variance condition 
		$
		\E[\norm{\nabla f(x;\xi)-\nabla f(x)}^2] \le \sigma^2
		$
		as special cases. The latter is because, by unbiasedness of $\nabla f(x;\xi)$, bounded variance is equivalent to 
		$
		\E[\norm{\nabla f(x;\xi)}^2] \le \norm{\nabla f(x)}^2+ \sigma^2.
		$
		Furthermore, in the context of solving stochastic or empirical minimization problems using SGD, by assuming that each realization or individual loss function is $L$-smooth and convex and that the overall objective function $f$ is strongly convex with a unique minimizer $x_*$, the following bound can be derived \citep{nguyen2019new}:
		$
		\E[\norm{\nabla f(x;\xi)}^2] \le 4L(f(x)-f^*) + \sigma^2,
		$
		where $\sigma^2=\E[\nabla f(x_*;\xi)]$. If the convexity condition on individual realization or loss function was dropped, a similar bound can still be shown \citep{nguyen2019new} with $4L$ replaced with $\frac{4L^2}{\mu}$. 
		Both of them are again special cases of the condition in Assumption \ref{as:abc}. For these reasons, we shall use the seemingly most general condition in Assumption \ref{as:abc} throughout this paper. 
	\end{remark}
	
	\section{Lemmas on supermartingale convergence rates}
	Our almost sure convergence rate analysis relies on the following classical supermartingale convergence theorem \citep{robbins1971convergence}.
	
	\begin{proposition}\label{prop:super}
		Let $\set{X_t}$, $\set{Y_t}$, and $\set{Z_t}$ be three sequences of random variables that are adapted to a filtration $\set{\mathcal{F}_t}$. Let $\set{\gamma_t}$ be a sequence of nonnegative real numbers such that $\Pi_{t=1}^\infty(1+\gamma_t)<\infty$. Suppose that the following conditions hold:
		\begin{enumerate}
			\item $X_t$, $Y_t$, and $Z_t$ are nonnegative for all $t\ge 1$. 
			\item $\mathbb{E}[Y_{t+1}\,\vert\, \mathcal{F}_t]\le (1+\gamma_t)Y_t - X_t + Z_t$ for all $t\ge 1$. 
			\item $\sum_{t=1}^{\infty}{Z_t}<\infty$ holds almost surely. 
		\end{enumerate}
		Then $\sum_{t=1}^{\infty}{X_t}<\infty$ almost surely and $Y_t$ converges almost surely. 
	\end{proposition}
	
	The following lemma, as a corollary of Proposition \ref{prop:super}, provides concrete estimates of almost sure convergence rates for sequences of random variables satisfying a supermartingale property. 
	
	\begin{lemma}\label{lem:strong}
		If $\set{Y_t}$ is a sequence of nonnegative random variables satisfying 
		\begin{equation}\label{ineq:Yt}
			\mathbb{E}[Y_{t+1}\,\vert\, \mathcal{F}_t]\le (1-c_1\alpha_t)Y_t + c_2\alpha_t^2,
		\end{equation}
		for all $t\ge 1$, where $\alpha_t=\Theta\left(\frac{1}{t^{1-\theta}}\right)$ for some $\theta\in (0,\frac12)$, and $c_1$ and $c_2$ are positive constants. Then, for any $\eps\in (2\theta,1)$, 
		$$
		Y_t = o\left(\frac{1}{t^{1-\eps}}\right),\quad \text{almost surely.}
		$$
	\end{lemma}
	
	The proof of Lemma \ref{lem:strong} can be found in Appendix \ref{app:lem:strong}. The following lemma, when used together with Proposition \ref{prop:super}, is useful for almost sure convergence rate analysis in a slightly different setting than Lemma \ref{lem:strong}. 
	
	\begin{lemma}\label{lem:weak}
		Let $\set{X_t}$ a sequence of nonnegative real numbers and $\set{\alpha_t}$ be a decreasing sequence of  positive real numbers such that the following holds:
		$$
		\sum_{t=1}^\infty \alpha_t X_t<\infty\quad\text{and}\quad \sum_{t=1}^\infty \frac{\alpha_{t}}{\sum_{i=1}^{t-1}\alpha_i}=\infty. 
		$$ 
		Define
		$
		w_t = \frac{2\alpha_t}{\sum_{i=1}^t\alpha_i},
		$
		$Y_1=X_1$, and 
		\begin{equation}\label{eq:Yt}
			Y_{t+1} = (1-w_t)Y_{t} + w_t X_t,\quad t\ge 1.
		\end{equation}
		Then 
		\begin{equation}
			\label{eq:YX}
			Y_t = o\left( \frac{1}{\sum_{i=1}^{t-1}\alpha_i} \right)\quad \text{and}\quad  \min_{1\le i\le t} X_i = o\left( \frac{1}{\sum_{i=1}^{t-1}\alpha_i} \right). 
		\end{equation}
	\end{lemma}
	
	\begin{remark}\label{rem:rate}
		A concrete convergence rate $o\left( \frac{1}{t^{\frac12-\eps}} \right)$ results from (\ref{eq:YX}) if we choose $\alpha_t = \frac{\alpha}{t^{\frac12+\eps}}$ for some $\alpha>0$ and $\eps\in(0,\frac12)$, because then we have  $\sum_{i=1}^{t-1}\alpha_i=\Theta(t^{\frac12-\eps})$, $\frac{\alpha_t}{\sum_{i=1}^{t-1}\alpha_i}=\Theta(\frac1t)$, and  $\sum_{t}\frac{\alpha_t}{\sum_{i=1}^{t-1}\alpha_i}=\infty$. 
	\end{remark}
	
	The proof of Lemma \ref{lem:weak} can be found in Appendix \ref{app:lem:weak}. 
	
	\section{Almost sure convergence rate analysis for stochastic gradient methods}
	
	In this section, we present a unified almost sure convergence rate analysis for SGD, SHB and SNAG. We primarily focus on two scenarios, namely the strongly convex and non-convex cases. 
	
	\subsection{Stochastic gradient descent}
	
	The iteration of the SGD method is given by
	\begin{equation}\label{eq:sgd}
		x_{t+1} = x_t - \alpha_t g_t, \quad t\ge 1, 
	\end{equation}
	where $g_t:=\nabla f(x_t;\xi_t)$ is the stochastic gradient at $x_t$ (with randomness $\xi_t$) and $\alpha_t$ is the step size. 
	
	We shall prove that, for smooth and strongly convex objective functions, SGD can achieve $o\left(\frac{1}{t^{1-\eps}}\right)$ almost sure convergence rates for any $\eps\in(0,1)$. To the best knowledge of the authors, this is the first result showing the $o\left(\frac{1}{t^{1-\eps}}\right)$ almost sure convergence rate for SGD under the global strong convexity assumption and relaxed assumption on stochastic gradients \citep{khaled2020better}. For smooth and non-convex objective functions, the best iterates of SGD can achieve $o\left(\frac{1}{t^{\frac12-\eps}}\right)$ almost sure convergence rates for any $\eps\in(0,\frac12)$. This result was already reported in \cite{sebbouh2021almost}. For locally strongly convex functions, similar rates were obtained in \cite{pelletier1998almost,godichon2019lp}. Here we provide a somewhat more streamlined proof of both the strongly convex and non-convex cases, enabled by Lemmas \ref{lem:strong} and \ref{lem:weak}. These rates match the lower bounds in \cite{agarwal2012information} (see also \cite{nemirovskij1983problem}) to an $\eps$-factor. 
	
	\begin{theorem}\label{thm:sgd}Consider the iterates of SGD (\ref{eq:sgd}). 
		\begin{enumerate} 
			\item If Assumptions \ref{as:smoothness}, \ref{as:strong}, and \ref{as:abc} hold and  $\alpha_t=\Theta\left(\frac{1}{t^{1-\theta}}\right)$ for some $\theta\in(0,\frac12)$, then almost surely
			$$
			f(x_t)-f^* = o\left(\frac{1}{t^{1-\eps}}\right),\quad \forall \eps\in(2\theta,1). 
			$$
			\item If Assumptions \ref{as:smoothness} and \ref{as:abc} hold and $\set{\alpha_t}$ is a decreasing sequence of positive real numbers satisfying 
			$
			\sum_{t=1}^\infty \frac{\alpha_t}{\sum_{i=1}^{t-1}\alpha_i}=\infty, 
			$
			then almost surely
			\begin{equation}\label{minX}
				\min_{1\le i\le t-1} \norm{\nabla f(x_t)}^2 = o\left( \frac{1}{\sum_{i=1}^{t-1}\alpha_i} \right). 
			\end{equation}
			In particular, if we choose $\alpha_t = \frac{\alpha}{t^{\frac12+\eps}}$ with $\alpha>0$ and $\eps\in (0,\frac12)$, 
			then almost surely 
			\begin{equation}\label{minX2}
				\min_{1\le i\le t-1} \norm{\nabla f(x_t)}^2 = o\left( \frac{1}{t^{\frac12-\eps}} \right). 
			\end{equation}
		\end{enumerate}

	\end{theorem}
	
	\begin{proof}
		1. We first consider the strongly convex case. By smoothness of $f$ and (\ref{eq:L}), we have
		\begin{align*}
			f(x_{t+1}) \le f(x_t) - \alpha_t\inner{\nabla f(x_t),g_t} + \frac{L\alpha_t^2}{2} \norm{g_t}^2.
		\end{align*}
		Taking conditional expectation w.r.t. $x_t$, denoted by $\E_t[\cdot]:=\E[\cdot \vert x_t]$, and using (\ref{eq:fmu}) lead to
		\begin{align}
			\E_t \left[f(x_{t+1})-f^*\right] &\le f(x_t)-f^* - \alpha_t \norm{\nabla f(x_t)}^2 + \frac{L\alpha_t^2}{2}\left[A(f(x_t)-f^*)+B\norm{\nabla f(x_t)}^2+C\right]\notag\\
			&= (1+\frac{LA\alpha_t^2}{2})(f(x_t)-f^*) - (\alpha_t-\frac{LB\alpha_t^2}{2})\norm{\nabla f(x_t)}^2 + \frac{LC\alpha_t^2}{2}\notag\\
			&\le (1+\frac{LA\alpha_t^2}{2})(f(x_t)-f^*) - 2\mu(\alpha_t-\frac{LB\alpha_t^2}{2})(f(x_t)-f^*) + \frac{LC\alpha_t^2}{2}\notag\\
			&= (1-2\mu\alpha_t+(LA/2+LB\mu)\alpha_t^2)(f(x_t)-f^*)  + \frac{LC\alpha_t^2}{2} \notag\\
			&\le (1-\mu\alpha_t)(f(x_t)-f^*)  + \frac{LC\alpha_t^2}{2}, \label{sgd:est1}
		\end{align}
		provided that $(LA/2+LB\mu)\alpha_t\le \mu$. The conclusion follows from Lemma \ref{lem:strong}. 
		
		2. For the non-convex case, by $L$-smoothness and as in (\ref{sgd:est1}), we obtain 
		\begin{align}
			\E_t \left[f(x_{t+1})-f^*\right] &\le f(x_t)-f^* - \alpha_t \norm{\nabla f(x_t)}^2 + \frac{L\alpha_t^2}{2}\left[A(f(x_t)-f^*)+B\norm{\nabla f(x_t)}^2+C\right]\notag\\
			&\le (1+\frac{LA\alpha_t^2}{2})(f(x_t)-f^*) - \left(\alpha_t -\frac{LB\alpha_t^2}{2} \right)\norm{\nabla f(x_t)}^2 + \frac{LC\alpha_t^2}{2} \notag\\
			&\le (1+\frac{LA\alpha_t^2}{2})(f(x_t)-f^*) - \frac{1}{2}\alpha_t\norm{\nabla f(x_t)}^2 + \frac{LC\alpha_t^2}{2}, \label{sgd:est2}
		\end{align}
		provided that $LB\alpha_t\le 1$. By Proposition \ref{prop:super}, $\sum_{t=1}^\infty \alpha_t\norm{\nabla f(x_t)}^2<\infty$. The conclusions follow from Lemma \ref{lem:weak} and Remark \ref{rem:rate}. 
	\end{proof}
	
	\begin{remark}
		We choose $\alpha_t=\Theta\left(\frac{1}{t^{1-\theta}}\right)$ 
		for $\theta\ra 0$ to approach the optimal almost sure convergence rate achievable under Lemma \ref{lem:strong}. In fact, any step size choice satisfying the classical condition by \cite{robbins1971convergence}: 
		$
		\sum_{t=1}^\infty {\alpha_t}=\infty
		$
		and 
		$ 
		\sum_{t=1}^\infty {\alpha_t^2}<\infty 
		$
		will lead to almost sure convergence under the supermartingale convergence theorem (Proposition \ref{prop:super}). What is new here is the analysis of almost sure convergence rate $o\left(\frac{1}{t^{1-\eps}}\right)$ for strongly convex objective functions using Lemma \ref{lem:strong}. By choosing $\theta\ra 0$, we can make $\eps\ra 0$. The conditions $(LA/2+LB\mu)\alpha_t\le \mu$ and $LB\alpha_t\le 1$ in the proof can be easily satisfied for all $t\ge 1$, if we scale all $\alpha_t$'s by a constant, or for all $t$ sufficiently large due to the choice of $\alpha_t$. This difference is insignificant because in the latter case the analysis in the proof holds asymptotically and the same convergence rate follows.  
	\end{remark}
	
	\subsection{Stochastic heavy-ball method}
	
	The iteration of the SHB method is given by
	\begin{equation}\label{eq:shb1}
		x_{t+1} = x_t - \alpha_t g_t + \beta (x_t - x_{t-1}),
	\end{equation}
	where $g_t:=\nabla f(x_t;\xi_t)$ is the stochastic gradient at $x_t$, $\alpha_t$ is the step size, and $\beta\in [0,1)$. Clearly, if $\beta=0$, SHB reduces to SGD. 
	
	Define
	\begin{equation}\label{eq:zv}
		z_t = x_t - \frac{\beta}{1-\beta} v_t,\quad v_t = x_t -x_{t-1}.
	\end{equation}
	The iteration of SHB can be rewritten as 
	\begin{equation}\label{eq:shb2}
		\begin{aligned}
			v_{t+1} & = \beta v_t -\alpha_t g_t,\\
			z_{t+1} & = z_t - \frac{\alpha_t}{1-\beta} g_t.
		\end{aligned}
	\end{equation}
	
	To our best knowledge, the following theorem provides the first almost sure convergence rates for SHB under both strongly convex and non-convex assumptions. 
	
	\begin{theorem}\label{thm:shb}
		Consider the iterates of SHB (\ref{eq:shb1}).
		\begin{enumerate} 
			\item If Assumptions \ref{as:smoothness}, \ref{as:strong}, and \ref{as:abc} hold and  $\alpha_t=\Theta\left(\frac{1}{t^{1-\theta}}\right)$ for some $\theta\in(0,\frac12)$, then almost surely 
			$$
			f(x_t)-f^* = o\left(\frac{1}{t^{1-\eps}}\right),\quad \forall \eps\in(2\theta,1). 
			$$
			\item If Assumptions \ref{as:smoothness} and \ref{as:abc} hold and $\set{\alpha_t}$ is a decreasing sequence of positive real numbers satisfying 
			$
			\sum_{t=1}^\infty \frac{\alpha_t}{\sum_{i=1}^{t-1}\alpha_i}=\infty, 
			$
			then almost surely 
			$$
			\min_{1\le i\le t-1} \norm{\nabla f(x_t)}^2 = o\left( \frac{1}{\sum_{i=1}^{t-1}\alpha_i} \right). 
			$$
			In particular, if we choose $\alpha_t = \frac{\alpha}{t^{\frac12+\eps}}$ with $\alpha>0$ and $\eps\in [0,\frac12]$, then almost surely  
			$$
			\min_{1\le i\le t-1} \norm{\nabla f(x_t)}^2 = o\left( \frac{1}{t^{\frac12-\eps}} \right).
			$$
		\end{enumerate}
	\end{theorem}
	
	\begin{proof} We have 
		\begin{align*}
			\norm{v_{t+1}}^2 = \beta^2 \norm{v_t}^2 - 2\beta\alpha_t \inner{g_t,v_t} + \alpha_t^2 \norm{g_t}^2. 
		\end{align*}
		Taking conditional expectation w.r.t. $x_t$, denoted by $\E_t[\cdot]:=\E[\cdot \vert x_t]$, gives
		\begin{align}
			&\E_t \norm{v_{t+1}}^2 = \beta^2 \norm{v_t}^2 - 2\beta\alpha_t \inner{\nabla f(x_t),v_t} + \alpha_t^2 \left[A(f(x_t)-f^*)+B\norm{\nabla f(x_t)}^2+C\right] \notag\\
			&\quad\le \beta^2 \norm{v_t}^2 + \eps_1 \norm{v_t}^2 + \frac{\alpha_t^2}{\eps_1}\norm{\nabla f(x_t)}^2 +  \alpha_t^2 \left[A(f(x_t)-f^*)+B\norm{\nabla f(x_t)}^2+C\right]. \label{shb:est1}
		\end{align}
		By $L$-smoothness of $f$ and (\ref{eq:L}), we have
		\begin{align*}
			f(z_{t+1}) \le f(z_t) - \frac{\alpha_t}{1-\beta}\inner{\nabla f(z_t),g_t} + \frac{L\alpha_t^2}{2(1-\beta)^2} \norm{g_t}^2.
		\end{align*}
		Taking conditional expectation w.r.t. $x_t$ gives
		\begin{align}
			&\E_t f(z_{t+1}) \notag\\
			&\le f(z_t) - \frac{\alpha_t}{1-\beta} \inner{\nabla f(z_t),\nabla f(x_t)}  + \frac{L\alpha_t^2}{2(1-\beta)^2}\left[A(f(x_t)-f^*)+B\norm{\nabla f(x_t)}^2+C\right]\notag\\
			&=f(z_t) -\frac{\alpha_t}{1-\beta} \norm{\nabla f(z_t)}^2- \frac{\alpha_t}{1-\beta} \inner{\nabla f(z_t),\nabla f(x_t)-\nabla f(z_t)}\notag\\
			&\qquad + \frac{L\alpha_t^2}{2(1-\beta)^2}\left[A(f(x_t)-f^*)+B\norm{\nabla f(x_t)}^2+C\right]\notag\\
			& \le f(z_t) -\frac{\alpha_t}{1-\beta} \norm{\nabla f(z_t)}^2 + \frac{\alpha_t}{1-\beta} \norm{\nabla f(z_t)}\frac{L\beta}{1-\beta}\norm{v_t}\notag\\
			&\qquad + \frac{L\alpha_t^2}{2(1-\beta)^2}\left[A(f(x_t)-f^*)+B\norm{\nabla f(x_t)}^2+C\right]\notag\\
			&\le f(z_t) -\frac{\alpha_t}{1-\beta} \norm{\nabla f(z_t)}^2 + \eps_2\norm{v_t}^2 + \frac{\alpha_t^2 L^2\beta^2}{\eps_2(1-\beta)^4}\norm{\nabla f(z_t)}^2\notag\\
			&\qquad + \frac{L\alpha_t^2}{2(1-\beta)^2}\left[A(f(x_t)-f^*)+B\norm{\nabla f(x_t)}^2+C\right].\label{shb:est2}
		\end{align}
		By $L$-smoothness of $f$ again, we have
		\begin{align}
			f(x_t)-f^* &\le f(z_t)-f^* +\frac{\beta}{1-\beta} \inner{\nabla f(z_t),v_t} + \frac{L\beta^2}{2(1-\beta)^2}\norm{v_t}^2\notag\\
			&\le f(z_t)-f^* +\frac{1}{2}\norm{\nabla f(z_t)}^2 + \frac{\beta^2}{2(1-\beta)^2} \norm{v_t}^2+ \frac{L\beta^2}{2(1-\beta)^2}\norm{v_t}^2,\label{shb:est3}
		\end{align}
		and
		\begin{align}
			\norm{\nabla f(x_t)}^2 &=  \norm{\nabla f(z_t)+ \nabla f(x_t)-\nabla f(z_t)}^2  \le 2\norm{\nabla f(z_t)}^2 + 2\norm{\nabla f(x_t)-\nabla f(z_t)}^2\notag\\
			& \le 2\norm{\nabla f(z_t)}^2 + 2\frac{L^2\beta^2}{(1-\beta)^2}\norm{v_t}^2.\label{shb:est4}
		\end{align}
		Combining (\ref{shb:est1})--(\ref{shb:est4}) yields
		\begin{align*}
			\E_t \left[f(z_{t+1}) -f^* + \norm{v_{t+1}}^2\right] &\le  (1+c_1\alpha_t^2)[f(z_t) - f^*] + (\beta^2 + \eps_1+\eps_2+c_2\alpha_t^2) \norm{v_t}^2 \notag \\
			&\qquad -\left( \frac{\alpha_t}{1-\beta} -c_3\alpha_t^2  \right)\norm{\nabla f(z_t)}^2 + c_4\alpha_t^2,
		\end{align*}
		where the constants $c_1$--$c_4$ can be straightforwardly determined from (\ref{shb:est1})--(\ref{shb:est4}). For any $\lambda\in(\beta,1)$, we can choose $\eps_1>0$ and $\eps_2>0$ such that $\beta^2 + \eps_1 +\eps_2\le \lambda$. For any $c\in (0,\frac{1}{1-\beta})$, we can choose $\alpha_t=\Theta\left(\frac{1}{t^{1-\theta}}\right)$, for some $\theta\in(0,\frac12)$, sufficiently small (by changing the constant) such that $\frac{\alpha_t}{1-\beta} -c_3\alpha_t^2\ge c\alpha_t$ for all $t\ge 1$. The above inequality becomes
		\begin{align}
			&\E_t \left[f(z_{t+1}) -f^* + \norm{v_{t+1}}^2\right] \notag\\
			&\qquad\le  (1+c_1\alpha_t^2)[f(z_t) - f^*] + (\lambda+c_2\alpha_t^2) \norm{v_t}^2 - c\alpha_t\norm{\nabla f(z_t)}^2 + c_4\alpha_t^2. \label{shb:est5}
		\end{align}
		
		We now consider two different cases: 
		
		1. If $f$ is $\mu$-strongly convex, we can use $\norm{\nabla f(z_t)}^2\ge 2\mu(f(z_t)-f^*)$ to further obtain 
		\begin{align*}
			\E_t \left[f(z_{t+1}) -f^* + \norm{v_{t+1}}^2\right] &\le  (1-2c\mu\alpha_t+c_1\alpha_t^2)[f(z_t) - f^*] + (\lambda+c_2\alpha_t^2) \norm{v_t}^2  + c_4\alpha_t^2. 
		\end{align*}
		By choosing $\alpha_t=\Theta\left(\frac{1}{t^{1-\theta}}\right)$ sufficiently small, the inequality leads to 
		\begin{align*}
			\E_t \left[f(z_{t+1}) -f^* + \norm{v_{t+1}}^2\right] &\le  (1-c_5\alpha_t)[f(z_t) - f^*+\norm{v_t}^2]   + c_4\alpha_t^2,
		\end{align*}
		for some constant $c_5>0$. It follows from Lemma \ref{lem:strong} that 
		$$
		f(z_{t+1}) -f^* + \norm{v_{t+1}}^2 = o\left(\frac{1}{t^{1-\eps}}\right)
		$$
		for any $\eps\in(2\theta,1)$. The conclusion follows from (\ref{shb:est3}) and (\ref{eq:Lf}).

		2. If $f$ is non-convex, by (\ref{shb:est4}), inequality (\ref{shb:est5}) leads to 
		\begin{align*}
			\E_t \left[f(z_{t+1}) -f^* + \norm{v_{t+1}}^2\right] &\le  (1+c_6\alpha_t^2)[f(z_t) - f^*+\norm{v_t}^2]- \frac{1}{2}c\alpha_t\norm{\nabla f(x_t)}^2 + c_4\alpha_t^2,
		\end{align*}
		where $c_6=\max(c_1,c_2)$, provided that $\alpha_t$ is chosen sufficiently small. By Proposition \ref{prop:super}, we have $\sum_{t=1}^\infty \alpha_t\norm{\nabla f(x_t)}^2<\infty$ almost surely. The conclusions follow from Lemma \ref{lem:weak} and Remark \ref{rem:rate}. 
	\end{proof}
	
	The almost sure convergence rates achieved by SHB are consistent with the best convergence rates possible for strongly convex and non-convex objective functions using stochastic gradient methods \citep{agarwal2012information} (see also \cite{nemirovskij1983problem}) subject to an $\eps$-factor. 
	
	\subsection{Stochastic Nesterov's accelerated gradient}
	The iteration of the SNAG method is given by
	\begin{equation}\label{eq:snag}
		\begin{aligned}
			y_{t+1} &= x_t - \alpha_t g_t,\\ 
			x_{t+1} &= y_{t+1} + \beta (x_t - x_{t-1}),
		\end{aligned}
	\end{equation}
	where $g_t:=\nabla f(x_t;\xi_t)$ is the stochastic gradient at $x_t$, $\alpha_t$ is the step size, and $\beta\in [0,1)$. Clearly, if $\beta=0$, SNAG also reduces to SGD. 
	
	Define $z_t$ and $v_t$ as in (\ref{eq:zv}). The iteration of SNAG can be rewritten as 
	\begin{equation}\label{eq:snag2}
		\begin{aligned}
			v_{t+1} & = \beta v_t -\beta \alpha_t g_t,\\
			z_{t+1} & = z_t - \frac{\alpha_t}{1-\beta} g_t.
		\end{aligned}
	\end{equation}
	Indeed, (\ref{eq:snag2}) is almost identical to (\ref{eq:shb2}) except for the extra $\beta$ in the first equation for $v_{t+1}$.  
	
	\begin{theorem}\label{thm:snag}
		Consider the iterates of SNAG (\ref{eq:snag}). 
		\begin{enumerate} 
			\item If Assumptions \ref{as:smoothness}, \ref{as:strong}, and \ref{as:abc} hold and $\alpha_t=\Theta\left(\frac{1}{t^{1-\theta}}\right)$ for some $\theta\in(0,\frac12)$, then almost surely 
			$$
			f(x_t)-f^* = o\left(\frac{1}{t^{1-\eps}}\right),\quad \forall \eps\in(2\theta,1). 
			$$
			\item If Assumptions \ref{as:smoothness} and \ref{as:abc} hold and $\set{\alpha_t}$ is a decreasing sequence of positive real numbers satisfying 
			$
			\sum_{t=1}^\infty \frac{\alpha_t}{\sum_{i=1}^{t-1}\alpha_i}=\infty, 
			$
			then almost surely 
			$$
			\min_{1\le i\le t-1} \norm{\nabla f(x_t)}^2 = o\left( \frac{1}{\sum_{i=1}^{t-1}\alpha_i} \right). 
			$$
			In particular, if we choose $\alpha_t = \frac{\alpha}{t^{\frac12+\eps}}$ with $\alpha>0$ and $\eps\in [0,\frac12]$, then almost surely 
			$$
			\min_{1\le i\le t-1} \norm{\nabla f(x_t)}^2 = o\left( \frac{1}{t^{\frac12-\eps}} \right).
			$$
		\end{enumerate}
	\end{theorem}
	
	\begin{proof}
		The proof is similar to that for Theorem \ref{thm:shb}. Instead of (\ref{shb:est1}), we obtain 
		\begin{align}
			\E_t \norm{v_{t+1}}^2 &\le \beta^2 \left(\norm{v_t}^2 + \eps_1 \norm{v_t}^2 + \frac{\alpha_t^2}{\eps_1}\norm{\nabla f(x_t)}^2 +  \alpha_t^2 \left[A(f(x_t)-f^*)+B\norm{\nabla f(x_t)}^2+C\right]\right). \label{snag:est1}
		\end{align}
		The rest of the proof proceeds in the same way (with slightly different constants). We conclude the same convergence rates by Lemmas \ref{lem:strong} and \ref{lem:weak}. 
	\end{proof}
	
	To our best knowledge, the above theorem provides the first result on almost sure convergence rates for SNAG under both strongly convex and non-convex assumptions. It is also evident from the above proofs that we provide a unified treatment the convergence analysis for SHB and SNAG.

	\section{Last-iterate convergence analysis of stochastic gradient methods} 
	
	In the previous sections, we have established close-to-optimal almost sure convergence rates for popular stochastic gradient methods. These rates are proved for the last iterate\footnote{Similar rates can be easily obtained for $\norm{x_t-x_*}^2$ and $\norm{\nabla f(x_t)}^2$ using strong convexity.} $f(x_t)-f^*$. When strong convexity is absent, convergence (rates) analysis for stochastic gradient methods in terms of the last iterates seems more challenging, even for general convex objective functions. We shall address these issues in this section. Such results are practically relevant, because it is the last iterates of gradient descent methods that are being used in most practical situations.
	
	\subsection{Last-iterate convergence analysis of SHB and SNAG for non-convex functions}
	In the non-convex setting, the convergence analysis in the previous sections shows that a weighted average of the squared gradient norm $\norm{\nabla f(x_i)}^2$ converges to zero almost surely, which also implies that the ``best'' iterate $\min_{1\le i\le t}\norm{\nabla f(x_i)}^2$ converges to zero almost surely (cf. Lemma \ref{lem:weak}). It is both theoretically intriguing and practically relevant to know whether the last-iterate gradient $\nabla f(x_t)$ converges almost surely. However, it is usually more challenging to analyze the convergence of the last iterate of SGD. An interesting discussion was made in \cite{orabona2020almost}, where the author simplified the long analysis in earlier work by \cite{bertsekas2000gradient} that proved the last-iterate $\norm{\nabla f(x_t)}^2$ converges almost surely to zero for SGD. In this section, we extend this analysis and prove that the last-iterate gradients of SHB and SNAG both converge to zero almost surely.

	We rely on the following lemma from \cite{orabona2020almost}, which can be seen as an extension of \citet[Proposition 2]{alber1998projected} and \citet[Lemma A.5]{mairal2013stochastic}. 
	
	\begin{lemma}[\cite{orabona2020almost}]\label{lem:orabona} Let $\set{b_t}$ and $\set{\alpha_t}$ be two nonnegative sequences and $\set{w_t}$ be a sequence of vectors. Assume $\sum_{t=1}^\infty\alpha_t b_t^p<\infty$ and $\sum_{t=1}^\infty\alpha_t=\infty$, where $p\ge 1$. Furthermore, assume that there exists some $L>0$ such that $\abs{b_{t+\tau}-b_t}\le L\left( \sum_{i=t}^{t+\tau-1}\alpha_i b_i + \norm{\sum_{i=t}^{t+\tau-1}\alpha_i w_i} \right)$, where $w_t$ is such that $\sum_{t=1}^\infty\alpha_t w_t$ converges. Then $b_t$ converges to 0. 
	\end{lemma}
	
	\begin{theorem}\label{thm:last-iterate}
		Consider the iterates of SHB (\ref{eq:shb1}) and SNAG (\ref{eq:snag}), respectively. Let Assumptions \ref{as:smoothness} and \ref{as:abc} hold and the step size $\set{\alpha_t}$ be a sequence of positive real numbers satisfying 
		$$
		\sum_{t=1}^\infty \alpha_t=\infty, \quad \sum_{t=1}^\infty \alpha_t^2 <\infty. 
		$$
		Then we have $\nabla f(x_t) \ra 0$ almost surely, as $t\ra \infty$, for both the iterates of SHB and SNAG. 
	\end{theorem}
	
	\begin{proof}
		We first prove that the last-iterate gradient of SHB converges. By (\ref{shb:est5}) and Proposition \ref{prop:super}, we have $\sum_{t=1}^\infty \alpha_t\norm{\nabla f(z_t)}^2<\infty$ almost surely. Furthermore, by $L$-smoothness of $f$, we have
		\begin{align*}
			\abs{\norm{\nabla f(z_{t+\tau})} - \norm{\nabla f(z_{t})} } &\le \norm{\nabla f(z_{t+\tau}) - \nabla f(z_{t})} \le  L \norm{z_{t+\tau} - z_t}  = L \norm{\sum_{i=t}^{t+\tau-1} \alpha_ig_i} \\
			& = L \norm{\sum_{i=t}^{t+\tau-1} \alpha_i \nabla f(z_i) + \alpha_i(g_i-\nabla f(z_i)) }\\
			& \le L\left( \sum_{i=t}^{t+\tau-1} \alpha_i \norm{\nabla f(z_i)} +  \norm{\sum_{i=t}^{t+\tau-1}\alpha_i w_i}  \right),
		\end{align*}
		where $w_i = g_i-\nabla f(z_i)$ .  To show that $\sum_{t\ge 0}\alpha_t w_t$ converges almost surely, we write 
		$$
		\alpha_t w_t = \alpha_t(g_t-\nabla f(x_t)) + \alpha_t(\nabla f(x_t) -\nabla f(z_t)). 
		$$
		
		We make the following claims that are proved in Appendix \ref{app:claims}. 
		
		\textbf{Claim 1:} $M_t=\sum_{i=1}^t \alpha_i (g_i-\nabla f(x_i))$ is a martingale bounded in $\mathcal{L}^2$ and hence converges almost surely \cite[Theorem 12.1]{williams1991probability}).

		\textbf{Claim 2:} $N_t=\sum_{i= 1}^t \alpha_i (\nabla f(x_i)-\nabla f(z_i))$ converges almost surely.

		By Claims 1 and 2, $\sum_{t=1}^\infty\alpha_t w_t$ converges almost surely. Applying Lemma \ref{lem:orabona} with $b_t=\norm{\nabla f(z_t)}$ and $p=2$ shows that $\nabla f(z_t)\ra 0$ almost surely. We conclude that $\nabla f(x_t)$ converges to 0 almost surely in view of (\ref{shb:est3}) and that $v_t\ra 0$ almost surely (since $\sum_{t=1}^\infty \norm{v_t}^2<\infty$ almost surely). 
		
		The proof of convergence for SNAG is similar, following (\ref{snag:est1}). We omitted the details here. 
	\end{proof}
	
	\subsection{Last-iterate convergence \textit{rates} of SGD, SHB, SNAG for general convex functions}
	
	We primarily focused on strongly convex and non-convex objective functions in the previous section. For functions that are generally convex, \cite{sebbouh2021almost} proved almost sure convergence rates of SGD for a weighted average of the iterates. A natural question to ask is whether one can obtain some last-iterate almost sure convergence rates. Indeed, the vast majority of convergence analysis for stochastic gradient methods under general convexity assumption yields results in terms of a weighted average of the iterates. There is an interesting discussion in \cite{orabona2020last}, where the author derived some last-iterate convergence rates in the context of non-asymptotic analysis for convergence in expectation (see also earlier work \cite{zhang2004solving,shamir2013stochastic} with more restricted domains or learning rates). In this section, we provide results on almost sure last-iterate convergence rates for SGD, SHB, and SNAG. Compared with the results in \cite{sebbouh2021almost} for SHB, we show that even without the iterate moving-average (IMA) parameter choices, the last iterates of SHB still converge to a minimizer, only assuming smoothness and convexity. We are not aware of any similar last-iterate almost sure convergence rates in the literature. 
	
	The proof of the following result can be found in Appendix \ref{app:last-iterate-rates}. 
	
	\begin{theorem}\label{thm:last-iterate-rates}
		Consider the iterates of SGD (\ref{eq:sgd}), SHB (\ref{eq:shb1}), and SNAG (\ref{eq:snag}), respectively. Suppose that we choose the step size $\alpha_t=\Theta\left( \frac{1}{t^{\frac23+\eps}}\right)$ for any $\eps\in (0,\frac13)$. Then we have $x\ra x_*$ for some $x_*$ such that $f(x_*)=f^*$ almost surely and 
		$f(x_t)-f^*=O\left(\frac{1}{t^{\frac13-\eps}}\right)$. 
	\end{theorem}

	\begin{remark}
		While this appears to the first result on last-iterate almost sure convergence rates for SGD and SNAG, the rate $O\left(\frac{1}{t^{\frac13-\eps}}\right)$ is not close to the lower bound obtained for convergence in expectation \citep{agarwal2012information}. Note that most convergence rates for SGD on general convex function are derived for a weighted average of the iterates. An interesting observation was made in   \cite{orabona2020last} and the author derived a non-asymptotic last-iterate convergence rate of $O\left(\frac{\log(T)}{\sqrt{T}}\right)$ in expectation. It is unclear at this point whether the idea in \cite{orabona2020last} can be extended to yield a close-to-optimal asymptotic almost sure convergence rate. It would be interesting to investigate whether the law of the  iterated logarithm for martingales \citep{stout1970martingale,de2004self,balsubramani2014sharp} can help determine the sharpest convergence rates in this setting. 
	\end{remark}
	
	\section{Conclusions}
	
	In this paper, we have provided a streamlined analysis of almost sure convergence rates for stochastic gradient methods, including SGD, SHB, and SNAG. The rates obtained for strongly convex functions are arbitrarily close to their corresponding optimal rates. For non-convex functions, the rates obtained for the \textit{best} iterates are close to the optimal convergence rates in expectation for general convex functions \citep{agarwal2012information}. For general convex functions, we identified a gap between the last-iterate almost sure convergence rates obtained and the possible optimal rates. Whether it is possible and how to close this gap can be an interesting topic for future work. 
	
	\acks{This work is partially supported by the NSERC Canada Research Chairs (CRC) program, an NSERC Discovery Grant, an Ontario Early Researcher Award (ERA), and the Jiangsu Industrial Technology Research Institute (JITRI) through a JITRI-Waterloo project.}

	\bibliography{colt22}
	
	\newpage 
	\appendix

	\section{Proof of Lemma \ref{lem:strong}}
	\label{app:lem:strong}
	
	\begin{proof}
		By the choice of $\alpha_t$, there exists some $\eta>0$ such that $c_1\alpha_t\ge \frac{\eta}{t^{1-\theta}}$ for all $t\ge 1$. 
		We shall make use of the elementary inequality 
		\begin{equation}
			\label{ieq:basic}
			(t+1)^{1-\eps}\le t^{1-\eps}+(1-\eps)t^{-\eps},
		\end{equation}
		which can be proved, for instance, as follows. Let $g(x)=x^{1-\eps}$. Then $g'(x)=(1-\eps)x^{-\eps}$ is decreasing. By the mean value theorem,  
		$$
		(t+1)^{1-\eps} - t^{1-\eps} = g'(\xi) \le g'(t) = (1-\eps)t^{-\eps},
		$$
		where $\xi \in (t,t+1)$, which implies inequality (\ref{ieq:basic}). Multiplying (\ref{ineq:Yt}) with $(t+1)^{1-\eps}$ and applying inequality (\ref{ieq:basic}) lead to
		\begin{align*}
			\mathbb{E}[(t+1)^{1-\eps}Y_{t+1}\,\vert\, \mathcal{F}_t] &\le (t+1)^{1-\eps}(1-c_1\alpha_t)Y_t + c_2(t+1)^{1-\eps}\alpha_t^2\\
			&\le  \left[t^{1-\eps}+(1-\eps)t^{-\eps}\right]\left(1-\frac{\eta}{t^{1-\theta}}\right) Y_t + c_2(t+1)^{1-\eps}\alpha_t^2 \\
			&= \left(1+\frac{1-\eps}{t}\right)\left(1-\frac{\eta}{t^{1-\theta}}\right) t^{1-\eps} Y_t+c_2(t+1)^{1-\eps}\alpha_t^2\\
			& = \left[1+\frac{1-\eps}{t}-\frac{\eta}{t^{1-\theta}}-\frac{\eta(1-\eps)}{t^{2-\theta}}\right] t^{1-\eps} Y_t+c_2(t+1)^{1-\eps}\alpha_t^2. 
		\end{align*}
		Clearly, as $t\ra\infty$, the dominating term in $\frac{1-\eps}{t}-\frac{\eta}{t^{1-\theta}}-\frac{\eta(1-\eps)}{t^{2-\theta}}$ is $-\frac{\eta}{t^{1-\theta}}$. Hence, there exists some $T>1$ sufficiently large such that, for all $t\ge T$, 
		\begin{align*}
			\mathbb{E}[(t+1)^{1-\eps}Y_{t+1}\,\vert\, \mathcal{F}_t] &\le   t^{1-\eps} Y_t - \frac{\eta}{2t^{1-\theta}}t^{1-\eps} Y_t  +c_2(t+1)^{1-\eps}\alpha_t^2. 
		\end{align*}
		With $\hat{Y}_t=t^{1-\eps}Y_{t}$, $X_t=\frac{\eta}{2t^{1-\theta}}t^{1-\eps} Y_t$, $Z_t=c_2(t+1)^{1-\eps}\alpha_t^2=\Theta\left(\frac{1}{t^{1+\eps-2\theta}}\right)$, and $\gamma_t=0$, the conditions of Proposition \ref{prop:super} are met for all $t\ge $T with $\hat{Y}_t$ in place of $Y_t$. By Proposition \ref{prop:super}, we have $t^{1-\eps} Y_t$ converges and $\sum_{t=T}^{\infty}{X_t}<\infty$ almost surely. We must have $t^{1-\eps} Y_t\ra 0$ almost surely, since $\sum_{t=T}^\infty \frac{\eta}{2t^{1-\theta}}=\infty$. The conclusion follows. 
	\end{proof}

	\section{Proof of Lemma \ref{lem:weak}}
	\label{app:lem:weak}
	
	\begin{proof}
		Note that $w_1=2$ and $Y_2=Y_1$. Since $\alpha_t$ is monotonically decreasing, $w_t\in [0,1]$ for $t\ge 2$. It follows that, for each $t\ge 2$, $Y_t$ is a weighted average of all numbers in $\set{X_1,\cdots,X_{t-1}}$. Furthermore, by (\ref{eq:Yt}) we have
		\begin{equation}\label{eq:Yt2}
			Y_{t+1}\sum_{i=1}^t\alpha_i = Y_{t}\sum_{i=1}^{t-1}\alpha_i - \alpha_t Y_t + 2\alpha_t X_t,\quad t\ge 1.
		\end{equation}
		Let $\hat{Y}_t=Y_{t}\sum_{i=1}^{t-1}\alpha_i$. Then conditions of Proposition \ref{prop:super} are met with $\hat{Y}_t$ in place of $Y_t$, $-\alpha_tY_t$ in place of $Y_t$, and $2\alpha_2X_t$ in place of $Z_t$. It follows from Proposition \ref{prop:super} that $Y_{t+1}\sum_{i=1}^t\alpha_i$ converges\footnote{While no random sequences are involved here, Proposition \ref{prop:super} is still applicable with almost sure convergence replaced by convergence. A direct proof is possible using the monotone convergence theorem for real numbers.} and $\sum_{t=1}^\infty \alpha_t Y_t <\infty$. Since $\sum_{t=1}^\infty \frac{\alpha_t}{\sum_{i=1}^{t-1}\alpha_i}=\infty$, $\sum_{t=1}^\infty \alpha_t Y_t <\infty$, and  $\lim_{t\ra\infty}\frac{\alpha_t Y_t}{\frac{\alpha_t}{\sum_{i=1}^{t-1}\alpha_i}}=\lim_{t\ra\infty}Y_{t}\sum_{i=1}^{t-1}\alpha_i$ exists, we must this limit equal 0 by the limit comparison test for series. Hence $Y_t = o\left( \frac{1}{\sum_{i=1}^{t-1}\alpha_i} \right)$. The other part of the conclusion follows by noting 
		$\min_{1\le i\le t-1} X_i \le Y_{t}$, because $Y_t$ is a weighted average of  $\set{X_1,\cdots,X_{t-1}}$. 
	\end{proof}
	
	\section{Proof of Theorem \ref{thm:last-iterate}} \label{app:claims}

	\textbf{Proof of Claim 1:} It is straightforward to verify by definition that it is a martingale. It is well known (see, e.g., \cite[Theorem 12.1]{williams1991probability}) that $M_t$ is bounded in $\mathcal{L}^2$ if and only if 
	$$
	\sum_{t=1}^\infty\E[\norm{M_t-M_{t-1}}^2] <\infty.
	$$  
	The latter is verified by
	\begin{align}
		\sum_{t=1}^\infty\E[\norm{M_t-M_{t-1}}^2]  &= \sum_{t=1}^\infty \alpha_t^2 (\E \norm{g_t}^2 - \E\norm{\nabla f(x_t)}^2)\notag \\
		&\le \sum_{t=1}^\infty \alpha_t^2 \left[A (\E[f(x_t)] - f^*) + (B-1) \E\norm{\nabla f(x_t)}^2 + C\right], \label{eq:em}
	\end{align}
	where we used Assumption \ref{as:abc}. Following the same argument as in the proof of Theorem \ref{thm:shb}, except that we take expectation on all the inequalities involved, we can show that $\E[f(x_t)] - f^*$ converges as $t\ra\infty$ and $\sum_{t=1}^\infty \alpha_t \E\norm{\nabla f(x_t)}^2<\infty$.  Since $\sum_{t=1}^\infty \alpha_t^2<\infty$, we have $\alpha_t\ra 0$ as $t\ra 0$.  By comparing the series on the right-hand side of (\ref{eq:em}) with convergent series $\sum_{t=1}^\infty \alpha_t^2$ and $\sum_{t=1}^\infty \alpha_t \E\norm{\nabla f(x_t)}^2$, we conclude that $\sum_{t=1}^\infty\E[\norm{M_t-M_{t-1}}^2] <\infty$.

	\noindent\textbf{Proof of Claim 2:} By $L$-smoothness of $f$, we have 
	\begin{align}
		\sum_{i=1}^t \norm{\alpha_i (\nabla f(x_i)-\nabla f(z_i))}  & \le \sum_{i=1}^t  \alpha_i L \norm{x_i-z_i} = \frac{L\beta}{1-\beta}\sum_{i=1}^t  \alpha_i\norm{v_i}\\
		& \le \frac{L\beta}{1-\beta} \sqrt{\sum_{i=1}^t \alpha_i^2}\sqrt{\sum_{i=1}^t \norm{v_i}^2}. \label{eq:N}
	\end{align}
	It follows that $N_t$ converges almost surely, provided that $\sum_{t=1}^\infty \norm{v_t}^2<\infty$ almost surely. To show the latter, recall  (\ref{shb:est5}) as 
	\begin{align}
		\E_t \left[f(z_{t+1}) -f^* + \norm{v_{t+1}}^2\right] &\le  (1+c_6\alpha_t^2)[f(z_t) - f^* + \norm{v_t}^2]  - (1-\lambda)\norm{v_t}^2  \notag \\
		&\qquad - c\alpha_t\norm{\nabla f(z_t)}^2 + c_4\alpha_t^2, \label{shb:est6}
	\end{align}
	where $c_6=\max(c_1,c_2)$. Proposition \ref{prop:super} implies that $\sum_{t=1}^\infty \norm{v_t}^2<\infty$ almost surely.

	\section{Proof of Theorem \ref{thm:last-iterate-rates}}
	\label{app:last-iterate-rates}

	\begin{lemma}\label{lem:convex}
		Suppose that $Y_t$ is a sequence of nonnegative random variables that are adapted to a filtration $\set{\mathcal{F}_t}$. Let $\set{\alpha_t}$ be a sequence chosen as $\alpha_t=\Theta\left( \frac{1}{t^{\frac23+\eps}}\right)$ (for $t\ge 1$), where $\eps\in (0,\frac13)$. If 
		\begin{align}
			\mathbb{E}[Y_{t+1}\,\vert\, \mathcal{F}_t] &\le (1+c_1\alpha_t^2)Y_t + c_2\alpha_t^2, \label{Y:est}
		\end{align}	
		for some constants $c_1,c_2>0$ and $\sum_{t=1}^\infty \alpha_t Y_t<\infty$ almost surely, then $Y_t=O\left(\frac{1}{t^{\frac13-\eps}}\right)$ almost surely. 
	\end{lemma}
	
	\begin{proof} Suppose that 
		$$
		\frac{\eta_1}{t^{\frac23+\eps}}\le \alpha_t\le\frac{\eta_2}{t^{\frac23+\eps}},\quad \forall t\ge 1,
		$$ 
		with some positive constants $\eta_1$ and $\eta_2$. Multiplying both sides of (\ref{Y:est}) by $(1+t)^{\frac13-\eps}$ leads to 
		\begin{align*}
			\E_t [(1+t)^{\frac13-\eps} Y_{t+1} \,\vert\, \mathcal{F}_t] & \le (1+t)^{\frac13-\eps}(1+c_1\alpha_t^2)Y_t + c_2(1+t)^{\frac13-\eps}\alpha_t^2\\
			&\le \left[t^{\frac13-\eps} + \left(\frac13-\eps\right)t^{-\frac23-\eps}\right](1+c_1\alpha_t^2)Y_t + c_2(1+t)^{\frac13-\eps}\alpha_t^2\\
			& = (1+c_1\alpha_t^2)t^{\frac13-\eps} Y_t  + \left(\frac13-\eps\right)t^{-\frac23-\eps}(1+c_1\alpha_t^2) Y_t + c_2(1+t)^{\frac13-\eps}\alpha_t^2 \\
			& \le (1+c_1\alpha_t^2)t^{\frac13-\eps} Y_t +(c_1\eta_2^2+1)\left(\frac13-\eps\right)t^{-\frac23-\eps}Y_t + \frac{c_2\eta_2^2}{t^{1+3\eps}}\frac{(1+t)^{\frac13-\eps}}{t^{\frac13-\eps}}\\
			&\le (1+c_1\alpha_t^2)t^{\frac13-\eps} Y_t +c_3\alpha_t Y_t + \frac{c_4}{t^{1+3\eps}},
		\end{align*}
		where we can take $c_3=(c_1\eta_2^2+1)\left(\frac13-\eps\right)/\eta_1$ and $c_4=c_2\eta_2^2\sqrt[3]{2}$. Recall that $\sum_{t=1}^\infty \alpha_t Y_t<\infty$. Applying Proposition \ref{prop:super} with $t^{\frac13-\eps} Y_t$ in place of $Y_t$, $X_t=0$, and $Z_t=c_3\alpha_t Y_t + \frac{c_4}{t^{1+3\eps}}$, we have $\sum_{t=1}^\infty Z_t< \infty$ and $t^{\frac13-\eps} Y_t$ converges almost surely. The conclusion follows. 
	\end{proof}
	
	With this lemma, we are ready to present the proof of Theorem \ref{thm:last-iterate-rates}. 
	
	\begin{proof} 1) We show the proof for SGD first. By smoothness of $f$ and (\ref{eq:L}), we have
		\begin{align*}
			f(x_{t+1}) \le f(x_t) - \alpha_t\inner{\nabla f(x_t),g_t} + \frac{L\alpha_t^2}{2} \norm{g_t}^2.
		\end{align*}
		Taking conditional expectation w.r.t. $x_t$, denoted by $\E_t[\cdot]:=\E[\cdot \vert x_t]$, leads to
		\begin{align}
			\E_t \left[f(x_{t+1})-f^*\right] &\le f(x_t)-f^* - \alpha_t \norm{\nabla f(x_t)}^2 + \frac{L\alpha_t^2}{2}\left[A(f(x_t)-f^*)+B\norm{\nabla f(x_t)}^2+C\right]\notag\\
			&\le (1+\frac{LA\alpha_t^2}{2})(f(x_t)-f^*) - \left(\alpha_t -\frac{LB\alpha_t^2}{2} \right)\norm{\nabla f(x_t)}^2 + \frac{LC\alpha_t^2}{2} \notag\\
			&\le (1+\frac{LA\alpha_t^2}{2})(f(x_t)-f^*) - \frac{1}{2}\alpha_t\norm{\nabla f(x_t)}^2 + \frac{LC\alpha_t^2}{2}, \label{sgd:est11}
		\end{align}
		provided that $LB\alpha_t\le 1$. 
		
		Let $x_*$ be a minimizer, i.e., $f(x_*)=f^*$. We have
		$$
		\norm{x_{t+1}-x^*}^2 = \norm{x_t - x_*}^2 - 2\alpha_t \inner{g_t,x_t-x_*} + \alpha_t^2 \norm{g_t}^2. 
		$$
		Take conditional expectation w.r.t. $x_t$ from both side. By convexity of $f$, we obtain 
		\begin{align}
			\E_t \left[\norm{x_{t+1}-x^*}^2 \right] &= \norm{x_t - x_*}^2 - 2
			\alpha_t\inner{\nabla f(x_t),x_t-x_*}\notag\\
			&\qquad + \alpha_t^2 \left[A(f(x_t)-f^*)+B\norm{\nabla f(x_t)}^2+C\right]\notag\\
			&\le \norm{x_t - x_*}^2  - 2\alpha_t\left(f(x_t)-f^* + \frac{1}{2L}\norm{\nabla f(x_t)}^2\right) \notag\\
			&\qquad + \alpha_t^2 \left[A(f(x_t)-f^*)+B\norm{\nabla f(x_t)}^2+C\right]\notag \\
			& = \norm{x_t - x_*}^2 - (2\alpha_t-A\alpha_t^2)(f(x_t)-f^*)-\left(\frac{1}{L}\alpha_t-B\alpha_t^2\right)\norm{\nabla f(x_t)}^2\notag\\
			&\qquad + \alpha_t^2 C \notag\\
			& \le \norm{x_t - x_*}^2 - \alpha_t (f(x_t)-f^*) + \alpha_t^2 C, \label{sgd:est22}
		\end{align}
		provided that $A\alpha_t\le 1$, in addition to $LB\alpha_t\le 1$. 
		
		By (\ref{sgd:est11}) and Proposition \ref{prop:super}, $\sum_{t=1}^\infty \alpha_t\norm{\nabla f(x_t)}^2 <\infty$ almost surely and $f(x_t)$ converges almost surely. By (\ref{sgd:est22}) and Proposition \ref{prop:super}, $\sum_{t=1}^\infty \alpha_t (f(x_t)-f^*) <\infty$ almost surely and $\norm{x_{t+1}-x^*}$ converges almost surely. Since $\sum_{t=1}^\infty \alpha_t=\infty$, we have  $f(x_t)$ converges to $f^*$ almost surely. By almost sure convergence of $\norm{x_{t+1}-x^*}$, $\set{x_t}$ almost surely has a convergent subsequence. The limit of this subsequence, denoted by $x(\omega)$ must satisfy $f(x(\omega))=f^*$. Hence $x(\omega)$ is also a minimizer. Since the choice of minimizer in (\ref{sgd:est22}) is arbitrary, we must have $x_t$ converges almost surely to some random variable. It follows that $\nabla f(x_t)$ exists almost and the limit must be 0 (either by using the fact that the limit of $x_t$ is a minimizer almost surely or that $\sum_{t=1}^\infty \alpha_t\norm{\nabla f(x_t)}^2 <\infty$ and $\sum_{t=1}^\infty \alpha_t=\infty$). 
		
		We now derive a concrete convergence rate for 
		$f(x_t)-f^*$. Let $Y_t=f(x_t)-f^*$. By (\ref{sgd:est11}) (and dropping the term $-\frac12\alpha_t\norm{\nabla f(x_t)}^2$), (\ref{Y:est}) of Lemma \ref{lem:convex} holds with $c_1=\frac{LA}{2}$ and $c_2=\frac{LC}{2}$. The conclusion follows from that of Lemma \ref{lem:convex}. 
		
		2) We now prove the case for SHB. Recall (\ref{shb:est5}) as 
		\begin{align}
			\E_t \left[f(z_{t+1}) -f^* + \norm{v_{t+1}}^2\right] &\le  (1+c_6\alpha_t^2)[f(z_t) - f^* + \norm{v_t}^2]  - (1-\lambda)\norm{v_t}^2  \notag \\
			&\qquad - c\alpha_t\norm{\nabla f(z_t)}^2 + c_4\alpha_t^2, \label{shb:est7}
		\end{align}
		where $c_6=\max(c_1,c_2)$ defined in (\ref{shb:est5}). Proposition \ref{prop:super} implies that $\sum_{t=1}^\infty \norm{v_t}^2<\infty$, $f(z_{t}) -f^*$ converges, and $\sum_{t=1}^\infty \alpha_t\norm{\nabla f(z_t)}^2<\infty$, almost surely. 
		
		Similar to (\ref{sgd:est22}), by convexity of $f$ and iterates of SHB in (\ref{eq:shb2}), we obtain
		\begin{align}
			\E_t \left[\norm{z_{t+1}-x^*}^2 \right] &= \norm{z_t - x_*}^2 - \frac{2
				\alpha_t}{1-\beta}\inner{\nabla f(x_t),z_t-x_*}\notag\\
			&\qquad  + \alpha_t^2 \left[A(f(x_t)-f^*)+B\norm{\nabla f(x_t)}^2+C\right]\notag\\
			&= \norm{z_t - x_*}^2 - \frac{2
				\alpha_t}{1-\beta}\inner{\nabla f(z_t),z_t-x_*} + \frac{2
				\alpha_t}{1-\beta}\inner{\nabla f(z_t)-f(x_t),z_t-x_*}\notag\\
			&\qquad  + \alpha_t^2 \left[A(f(x_t)-f^*)+B\norm{\nabla f(x_t)}^2+C\right]\notag\\
			&\le \norm{z_t - x_*}^2  - 2\alpha_t\left(f(z_t)-f^* + \frac{1}{2L}\norm{\nabla f(z_t)}^2\right) +\frac{\beta^2L^2}{(1-\beta)^4}\norm{v_t}^2\notag\\
			&\qquad + \alpha_t^2\norm{z_t-x_*}^2 + \alpha_t^2 \left[A(f(x_t)-f^*)+B\norm{\nabla f(x_t)}^2+C\right]\notag \\
			& \le (1+\alpha_t^2)\norm{z_t - x_*}^2 - (2\alpha_t-c_7\alpha_t^2)(f(z_t)-f^*)\notag\\
			&\qquad -\left(\frac{1}{L}\alpha_t-c_8\alpha_t^2\right)\norm{\nabla f(z_t)}^2 +c_9 \norm{v_t}^2 + \alpha_t^2 C, \label{sgd:est33}
		\end{align}
		where $c_7$, $c_8$, and $c_9$ are some positive constants. The first inequality above follows from convexity of $f$, $L$-Lipschitzness of $\nabla f$, and the elementary inequality $2\inner{a,b}\le \norm{a}^2+\norm{b}^2$. The second inequality follows from (\ref{shb:est3}). By (\ref{sgd:est33}), choosing $\alpha_t$ sufficiently small leads to  
		\begin{align}
			\E_t \left[\norm{z_{t+1}-x^*}^2 \right] \le (1+\alpha_t^2)\norm{z_t - x_*}^2 - \alpha_t(f(z_t)-f^* + \norm{v_t}^2) +c_{10} \norm{v_t}^2 + \alpha_t^2 C, \label{sgd:est44}
		\end{align}
		where $\alpha_t+c_9\le c_{10}$. Since $\sum_{t=1}^\infty \norm{v_t}^2<\infty$, Proposition \ref{prop:super} implies $\sum_{t=1}^\infty  \alpha_t (f(z_t)-f^* + \norm{v_t}^2)<\infty$ and $\norm{z_t - x_*}^2$ converges almost surely.  By a similar argument as in the proof for SGD, we have $z_t$ converges to a minimizer almost surely. To obtain a concrete convergence rate, let $Y_t=f(z_t)-f^* + \norm{v_t}^2$. By the choice of $\alpha_t$ and Lemma \ref{lem:convex}, we have
		$$
		Y_t = f(z_t)-f^* + \norm{v_t}^2 = O\left(\frac{1}{t^{\frac13-\eps}}\right). 
		$$
		From (\ref{shb:est3}) and the fact that
		$$
		\norm{\nabla f(z_t)}^2 \le 2L (f(z_{t})-f^*),
		$$
		we obtain 
		$$
		f(x_{t}) -f^* = O\left(\frac{1}{t^{\frac13-\eps}}\right). 
		$$
		
		3) The case for SNAG is very similar in view of (\ref{snag:est1}) and omitted. 
	\end{proof}

\end{document}